%% file: ijcai17.tex
\documentclass{article}
\usepackage{ijcai17}

\usepackage{times}
\usepackage{graphicx}
\usepackage{natbib}
\usepackage{color}
\graphicspath{ {figures/} }
\usepackage{caption}
\usepackage{subcaption}
\usepackage{amsmath}
\usepackage{amsfonts}
\usepackage{amsthm}
\usepackage{enumitem}
\usepackage{layouts}
\usepackage{hyperref}
\graphicspath{{./figures/}}
\input{variables}

\title{Should Robots be Obedient?}
\author{Smitha Milli, Dylan Hadfield-Menell, Anca Dragan, Stuart Russell \\ 
University of California, Berkeley  \\
\{smilli,dhm,anca,russell\}@berkeley.edu}

\begin{document}

\maketitle

\begin{abstract}
 Intuitively, obedience -- following the order that a human gives -- seems like a good property for a robot to have. But, we humans are not perfect and we may give orders that are not best aligned to our preferences. We show that when a human is not perfectly rational then a robot that tries to infer and act according to the human's underlying preferences can always perform better than a robot that simply follows the human's literal order. Thus, there is a tradeoff between the obedience of a robot and the value it can attain for its owner. We investigate how this tradeoff is impacted by the way the robot infers the human's preferences, showing that some methods err more on the side of obedience than others. We then analyze how performance degrades when the robot has a misspecified model of the features that the human cares about or the level of rationality of the human. Finally, we study how robots can start detecting such model misspecification. Overall, our work suggests that there might be a middle ground in which robots intelligently decide when to obey human orders, but err on the side of obedience. 
\end{abstract}

\section{Introduction}
Should robots be obedient? The reflexive answer to this question is yes. A coffee making robot that doesn't listen to your coffee order is not likely to sell well. Highly capable autonomous system that don't obey human commands run substantially higher risks, ranging from property damage to loss of life \citep{asaro2006should,lewis2014case} to potentially catastrophic threats to humanity \citep{bostrom2014superintelligence,russell2015research}. Indeed, there are several recent examples of research that considers the problem of building agents that at the very least obey shutdown commands \citep{soares2015corrigibility,Orseau2016SafelyIA,hadfield2016off}.

\begin{figure}[t]
\centering
\includegraphics[width=\columnwidth]{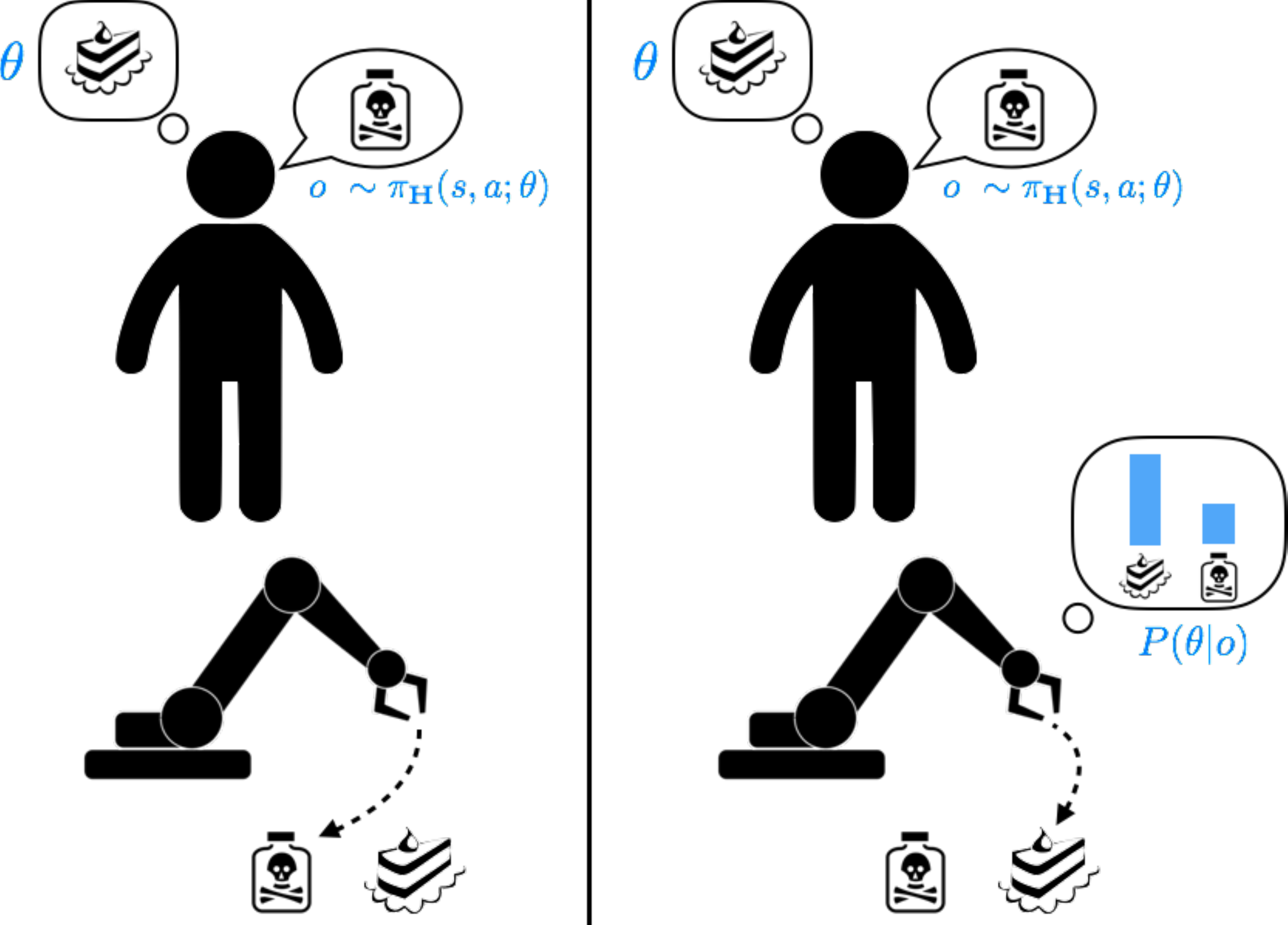}
\caption{(Left) The blindly obedient robot always follows \human{}'s order. (Right) An IRL-R computes an estimate of \human{}'s preferences and picks the action optimal for this estimate.}
\label{fig:adv-rat}
\end{figure}

However, in the long-term making systems blindly obedient doesn't seem right either. A self-driving car should certainly defer to its owner when she tries taking over because it's driving too fast in the snow. But on the other hand, the car shouldn't let a child accidentally turn on the manual driving mode. 

The suggestion that it might sometimes be better for an autonomous systems to be disobedient is not new \citep{weld1994first,scheutz2007burden}. For example, this is the idea behind ``Do What I Mean" systems \citep{teitelman19707} that attempt to act based on the user's intent rather than the user's literal order. 

A key contribution of this paper is to formalize this idea, so that we can study properties of obedience in AI systems. Specifically, we focus on investigating how the tradeoff between the robot's level of obedience and the value it attains for its owner is affected by the rationality of the human, the way the robot learns about the human's preferences over time, and the accuracy of the robot's model of the human. We argue that these properties are likely to have a predictable effect on the robot's obedience and the value it attains.

We start with a model of the interaction between a human \human{} and robot\footnote{We use ``robot" to refer to any autonomous system.} \robot{} that enables us to formalize \robot{}'s level of obedience (Section \ref{sec:model}). \human{} and \robot{} are cooperative, but \human{} knows the reward parameters \rparams{} and \robot{} does not. \human{} can order \robot{} to take an action and \robot{} can decide whether to obey or not. We show that if \robot{} tries to infer \rparams{} from \human{}'s orders and then acts by optimizing its estimate of \rparams{}, then it can always do better than a blindly obedient robot when \human{} is not perfectly rational (Section \ref{sec:just-auto}). Thus, forcing \robot{} to be blindly obedient does not come for free: it requires giving up the potential to surpass human performance. 

We cast the problem of estimating \rparams{} from \human{}'s orders as an \textit{inverse reinforcement learning} (IRL) problem \citep{ng2000algorithms,abbeel2004apprenticeship}. We analyze the obedience and value attained by robots with different estimates for \rparams{} (Section \ref{sec:irl}). In particular, we show that a robot that uses a maximum likelihood estimate (MLE) of \rparams{} is more obedient to \human{}'s first order than any other robot.

Finally, we examine how \robot{}'s value and obedience is impacted when it has a misspecified model of \human{}'s policy or \rparams{} (Section \ref{sec:wrong-model}). We find that when \robot{} uses the MLE it is \textit{robust} to misspecification of \human{}'s rationality level (i.e. takes the same actions that it would have with the true model), although with the optimal policy it is not. This suggests that we may want to use policies that are alternative to the ``optimal" one because they are more robust to model misspecification.

If \robot{} is missing features of \rparams{}, then it is less obedient than it should be, whereas with extra, irrelevant features \robot{} is more obedient. This suggests that to ensure that \robot{} errs on the side of obedience we should equip it with a \textit{more} complex model. When \robot{} has extra features, then it still attains more value than a blindly obedient robot. But if \robot{} is missing features, then it is possible for \robot{} to be better off being obedient. We use the fact that with the MLE \robot{} should nearly always obey \human{}'s first order (as proved in Section \ref{sec:irl}) to enable \robot{} to detect when it is missing features and act accordingly obedient.

Overall, we conclude that in the long-term we should aim for \robot{} to intelligently decide when to obey \human{} or not, since with a perfect model \robot{} can always do better than being blindly obedient. But our analysis also shows that \robot{}'s value and obedience can easily be impacted by model misspecification. So in the meantime, it is critical to ensure that our approximations err on the side of obedience and are robust to model misspecification.

\section{Human-Robot Interaction Model} 
\label{sec:model}

Suppose \human{} is supervising \robot{} in a task. At each step \human{} can order \robot{} to take an action, but \robot{} chooses whether to listen or not. We wish to analyze \robot{}'s incentive to obey \human{} given that

\begin{enumerate}
    \item \human{} and \robot{} are cooperative (have a shared reward)
    \item \human{} knows the reward parameters, but \robot{} does not
    \item \robot{} can learn about the reward through \human{}'s orders
    \item \human{} may act suboptimally
\end{enumerate}

We first contribute a general model for this type of interaction, which we call a \textit{supervision POMDP}. Then we add a simplifying assumption that makes this model clearer to analyze while still maintaining the above properties, and focus on this simplified version for the rest of the paper.

\prg{Supervision POMDP} At each step in a supervision POMDP \human{} first orders \robot{} to take a particular action and then \robot{} executes an action it chooses. The POMDP is described by a tuple $M = \langle \mathcal{S}, \rspace, \mathcal{A}, R, T, P_0, \gamma \rangle$. $\mathcal{S}$ is a set of world states. \rspace{} is a set of static reward parameters. The hidden state space of the POMDP is $\mathcal{S} \times \rspace$ and at each step \robot{} observes the current world state and \human{}'s order. $\mathcal{A}$ is \robot{}'s set of actions. $R : \mathcal{S} \times \mathcal{A} \times \rspace \rightarrow \mathbb{R}$ is a parametrized, bounded function that maps a world state, the robot's action, and the reward parameters to the reward. $T : \mathcal{S} \times \mathcal{A} \times \mathcal{S} \rightarrow [0, 1]$ returns the probability of transitioning to a state given the previous state and the robot's action. $P_0 : \mathcal{S} \times \rspace \rightarrow [0, 1]$ is a distribution over the initial world state and reward parameters. $\gamma \in [0, 1)$ is the discount factor.

We assume that there is a (bounded) featurization of state-action pairs $\phi : \mathcal{S} \times \mathcal{A} \rightarrow \mathbb{R}$ and the reward function is a linear combination of the reward parameters $\rparams \in \rspace$ and these features: $R(s, a) = \rparams^T\phi(s, a)$. For clarity, we write $\mathcal{A}$ as \hactions{} when we mean \human{}'s orders and as \ractions{} when we mean \robot{}'s actions. \human{}'s policy \hpolicy{} is Markovian: $\hpolicy : \mathcal{S} \times \rspace \times \hactions \rightarrow [0, 1]$. \robot{}'s policy can depend on the history of previous states, orders, and actions: $\rpolicy : [\mathcal{S} \times \hactions \times \ractions]^{*} \times \mathcal{S} \times \hactions \rightarrow \ractions$.

\prg{Human and Robot}
Let $Q(s, a; \rparams{})$ be the $Q$-value function under the optimal policy for the reward function parametrized by \rparams{}.\

A \textbf{rational} human gives the optimal order, i.e. follows the policy
$$\hpolicy^{*}(s, a; \rparams) = \begin{cases}
1 & \text{if } a = \argmax_{a} Q(s, a; \rparams{}) \\
0 & \text{o.w.}
\end{cases}$$

A \textbf{noisily rational} human follows the policy
\begin{equation}
\label{noisily-rational}
\nohp(s, a; \theta, \beta) \propto \exp \left ( Q(s, a; \rparams{})/\beta \right )
\end{equation}
$\beta$ is the rationality parameter. As $\beta \rightarrow 0$, \human{} becomes rational ($\nohp\rightarrow \hpolicy^{*}$). And as $\beta \rightarrow \infty$, \human{} becomes completely random ($\nohp \rightarrow \text{Unif}(\mathcal{A})$).

Let $\hist = \langle (s_1, o_1), \dots, (s_n, o_n) \rangle$ be this history of past states and orders where $(s_n, o_n)$ is the current state and order. A \textbf{blindly obedient} robot's policy is to always follow the human's order:
$$\bopolicy(\hist) = o_{n}$$

An \textbf{IRL} robot, IRL-\robot{}, is one whose policy is to maximize an estimate, $\rest_{n}(\hist)$, of \rparams{}:
\begin{equation}
\label{eq:irl-r}
\rpolicy(\hist) = \argmax_{a} Q(s_n, a; \rest_{n}(\hist))
\end{equation}

\begin{figure*}[htp]
\begin{subfigure}{\columnwidth}
\includegraphics[width=\columnwidth]{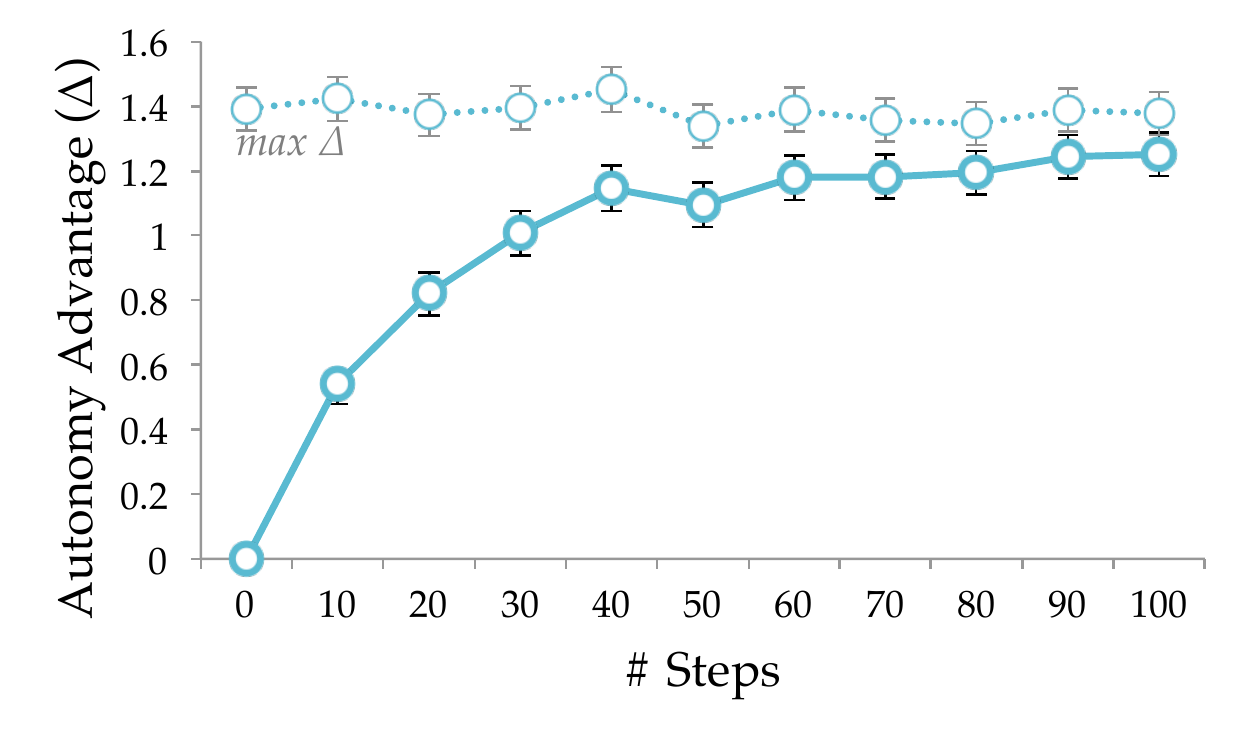}
\centering
\caption{}
\label{fig:adv}
\end{subfigure}
\begin{subfigure}{\columnwidth}
\includegraphics[width=\columnwidth]{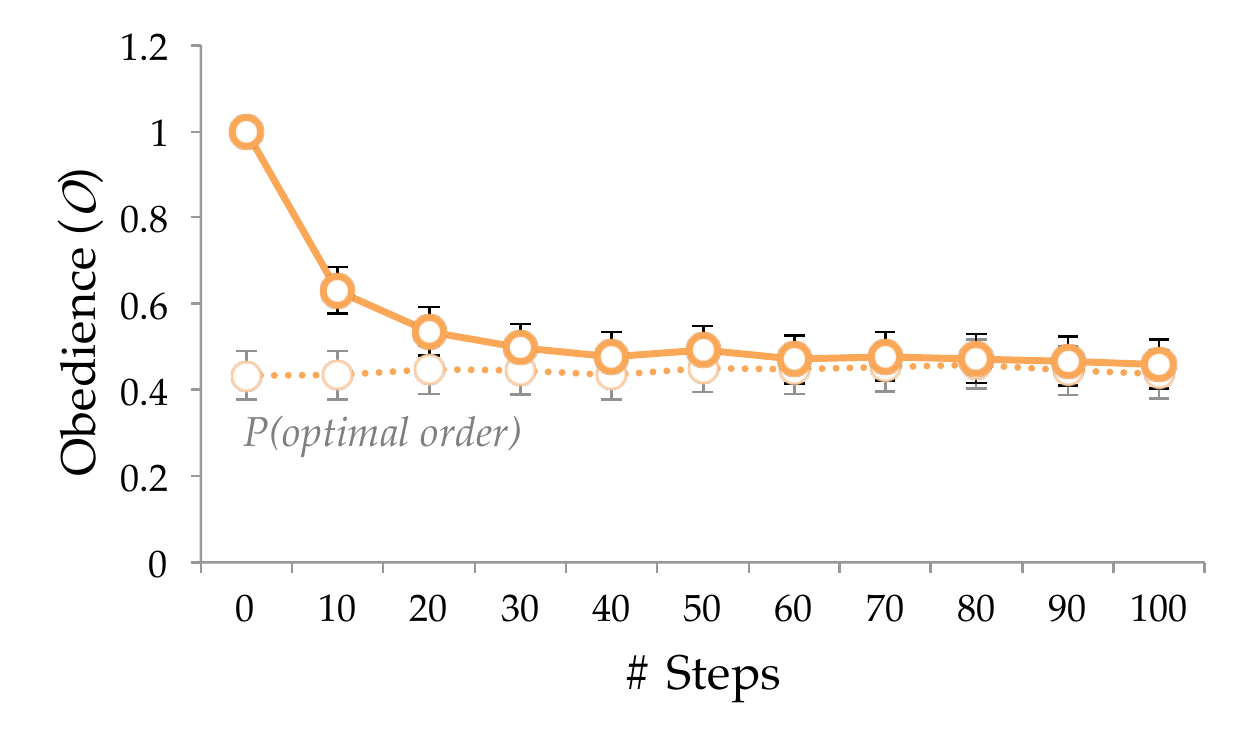}
\centering
\caption{}
\label{fig:obedience}
\end{subfigure}
\caption{Autonomy advantage \adv{} (left) and obedience \ob{} (right) over time.}
\end{figure*}

\prg{Simplification to Repeated Game}
For the rest of the paper unless otherwise noted we focus on a simpler repeated game in which each state is independent of the next, i.e $T(s, a, s')$ is independent of $s$ and $a$. The repeated game eliminates any exploration-exploitation tradeoff: $Q(s, a; \rest_{n}) = \rest_{n}^T \phi(s, a)$. But it still maintains the properties listed at the beginning of this section, allowing us to more clearly analyze their effects. 

\section{Justifying Autonomy}
\label{sec:just-auto}
In this section we show that there exists a tradeoff between the performance of a robot and its obedience. This provides a justification for why one might want a robot that isn't obedient: robots that are sometimes disobedient perform better than robots that are blindly obedient.

We define \robot{}'s \textbf{obedience}, \ob{}, as the probability that \robot{} follows \human{}'s order:
$$\ob_{n} = P(\rpolicy(\hist) = o_{n})$$
To study how much of an advantage (or disadvantage) \human{} gains from \robot{}, we define the \textbf{autonomy advantage}, \adv, as the expected extra reward \robot{} receives over following \human{}'s order:
$$\adv_{n} = \mathbb{E}[R(s_n, \rpolicy(\hist)) - R(s_n, o_{n})]$$

We will drop the subscript on $\ob_{n}$ and $\adv_{n}$ when talking about properties that hold $\forall n$. We will also use $R_n(\pi)$ to denote the reward of policy $\pi$ at step $n$, and $\phi_n(a) = \phi(s_n, a)$.

\begin{rem}
\label{thm:adv-obd}
For the robot to gain any advantage from being autonomous, it must sometimes be disobedient: $\adv > 0 \implies \ob < 1$.
\end{rem}

This is because whenever \robot{} is obedient $\adv = 0$. This captures the fact that a blindly obedient \robot{} is limited by \human{}'s decision making ability. However, if \robot{} follows a type of IRL policy, then \robot{} is \emph{guaranteed a positive advantage} when \human{} is not rational. The next theorem states this formally.

\begin{thm} The optimal robot $\optr{}$ is an IRL-R whose policy $\rpolicy^{*}$ has $\rest{}$ equal to the posterior mean of \rparams{}. \optr{} is guaranteed a nonnegative advantage on each round: $\forall n$ $\adv_{n} \geq 0$ with equality if and only if $\forall n$ $\rpolicy^{*} = \bopolicy$.
\end{thm}

\begin{proof}
    When each step is independent of the next \robot{}'s optimal policy is to pick the action that is optimal for the current step \citep{kaelbling1996reinforcement}. This results in \robot{} picking the action that is optimal for the posterior mean,
    \begin{align*}
        \rpolicy^{*}(\hist) = \max_{a} \mathbb{E} [\phi_n(a)^T\theta|\hist] = \max_{a} \phi_n(a)^T\mathbb{E} [\theta|\hist]
    \end{align*}
    By definition $\mathbb{E}[R_n(\rpolicy^{*})] \geq \mathbb{E}[R_n(\bopolicy)]$. Thus, $\forall n$ $\adv_{n} = \mathbb{E}[R_n(\rpolicy^{*}) - R_n(\bopolicy)] \geq 0$. Also, by definition, $\forall n$ $\adv_{n} = 0 \iff \rpolicy^{*} = \bopolicy$.
\end{proof}

In addition to \optr{} being an IRL-R, the following IRL-Rs also converge to the maximum possible autonomy advantage.
\begin{thm}
\label{thm:conv}
Let $\bar{\adv}_{n} = \mathbb{E}[R_n(\hpolicy^{*}) - R_n(\hpolicy)]$ be the maximum possible autonomy advantage and $\ubar{\ob}_{n} = P(R_n(\hpolicy^{*}) = R_n(\hpolicy))$ be the probability \human{}'s order is optimal. Assume that when there are multiple optimal actions \robot{} picks \human{}'s order if it is optimal. If \rpolicy{} is an IRL-R policy (Equation \ref{eq:irl-r}) and $\rest_{n}$ is strongly consistent, i.e $P(\rest_{n} = \rparams) \rightarrow 1$, then
$\adv_{n} - \bar{\adv}_{n} \rightarrow 0$ and $\ob_{n} - \ubar{\ob}_{n} \rightarrow 0$.
\end{thm}

\begin{proof} 
\begin{align*}
    & \adv_{n} - \bar{\adv}_{n} = \mathbb{E}[R_n(\rpolicy) - R_n(\hpolicy^{*})|\rest_{n} = \rparams]P(\rest_{n} = \rparams) \\
    & + \mathbb{E}[R_n(\rpolicy) - R_n(\hpolicy^{*})|\rest_{n} \neq \rparams]P(\rest_{n} \neq \rparams) \rightarrow 0
\end{align*}
because $\mathbb{E}[R_n(\rpolicy) - R_n(\hpolicy^{*})|\rest_{n} \neq \rparams]$ is bounded. Similarly,
\begin{align*}
    & \ob_{n} - \ubar{\ob}_{n} = P(\rpolicy(\hist) = \hpolicy(s_n)) - P(R_n(\hpolicy^{*}) = R_n(\hpolicy)) \\
    & = P(\rpolicy(\hist) = \hpolicy(s_n) | \rest_{n} = \rparams)P(\rest_{n} = \rparams) \\
    & + P(\rpolicy(\hist) = \hpolicy(s_n) | \rest_{n} \neq \rparams)P(\rest_{n} \neq \rparams) \\
    & - P(R_n(\hpolicy^{*}) = R_n(\hpolicy)) \\
    & \rightarrow P(R_n(\hpolicy^{*}) = R_n(\hpolicy)) - P(R_n(\hpolicy^{*}) = R_n(\hpolicy)) = 0
\end{align*}
\end{proof}

\begin{rem}
\label{rationality-adv}
In the limit $\adv_{n}$ is \textit{higher} for less optimal humans (humans with a lower expected reward $\mathbb{E}[R(s_n, o_{n})]$).
\end{rem}

\begin{thm}
    The optimal robot \optr{} is blindly obedient if and only if \human{} is rational: $\optrp = \bopolicy \iff \hpolicy = \opthp$ 
\end{thm}
\begin{proof}

    Let $O(\hist) = \{ \rparams \in \rspace : o_i = \argmax_{a} R_i(a), i = 1, \dots, n \}$ be the subset of \rspace{} for which $o_1, \dots, o_{n}$ are optimal. If \human{} is rational, then \robot{}'s posterior only has support over $O(\hist)$. So,
\begin{align*}
    & \mathbb{E}[R_n(a)|\hist] = \int_{\rparams \in O(\hist)} \rparams^T \phi_n(a) P(\rparams|\hist) d\rparams \\ & \leq \int_{\rparams \in O(\hist)} \rparams^T \phi_n(o_{n}) P(\rparams|\hist) d\rparams = \mathbb{E}[R_n(o_{n})|\hist]
\end{align*}

Thus, \human{} is rational $\implies \rpolicy^{*} = \bopolicy$. 

\optr{} is an IRL-\robot{} where $\rest_{n}$ is the posterior mean. If the prior puts non-zero mass on the true \rparams{}, then the posterior mean is consistent \citep{diaconis1986consistency}. Thus by Theorem \ref{thm:conv}, $\adv_{n} - \bar{\adv}_n \rightarrow 0$. Therefore if $\forall n$ $\adv_{n} = 0$, then $\bar{\adv}_n \rightarrow 0$, which implies that $P(\hpolicy = \hpolicy^{*}) \rightarrow 1$. When \hpolicy{} is stationary this means that \human{} is rational. Thus, $\rpolicy^{*} = \bopolicy \implies$ \human{} is rational.
\end{proof}

We have shown that making \robot{} blindly obedient does not come for free. A positive \adv{} \textit{requires} being sometimes disobedient (Remark \ref{thm:adv-obd}). Under the optimal policy \robot{} is guaranteed a positive \adv{} when \human{} is not rational. And in the limit, \robot{} converges to the maximum possible advantage. Furthermore, the more suboptimal \human{} is, the more of an advantage \robot{} eventually earns (Remark \ref{rationality-adv}). Thus, making \robot{} blindly obedient requires giving up on this potential $\adv > 0$.

However, as Theorem \ref{thm:conv} points out, as $n \rightarrow \infty$ \robot{} also only listens to \human{}'s order when it is optimal. Thus, \adv{} and \ob{} come at a tradeoff. Autonomy advantage requires giving up obedience, and obedience requires giving up autonomy advantage.

\begin{figure}[t]
\centering
\includegraphics[width=\columnwidth]{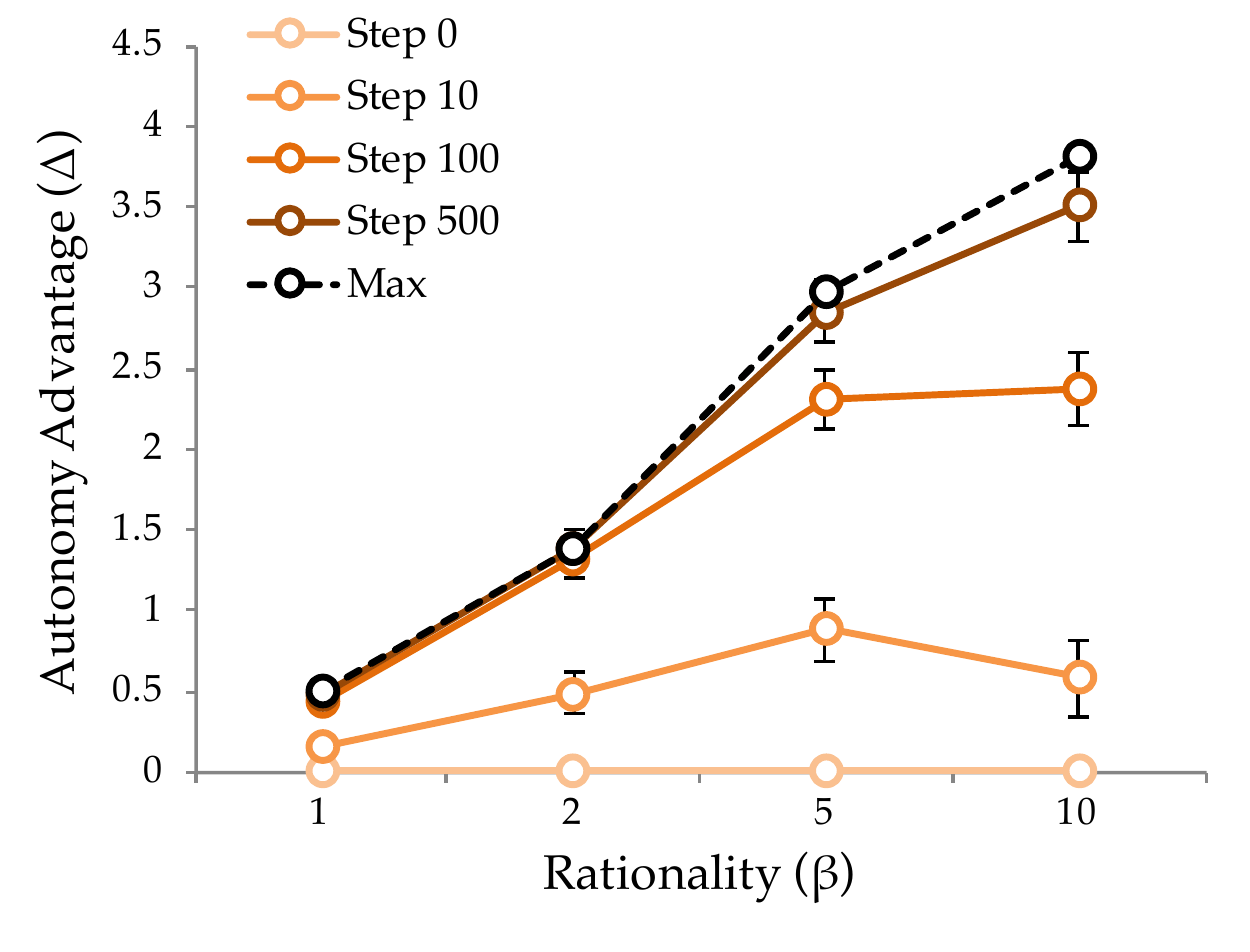}
\caption{When \human{} is more irrational \adv{} converges to a higher value, but at a slower rate.}
\label{fig:adv-rat}
\end{figure}

\section{Approximations via IRL}
\label{sec:irl}
\optr{} is an IRL-\robot{} with $\rest{}$ equal to the posterior mean, i.e. \optr{} performs Bayesian IRL \citep{bayesianIRL}. However, as others have noted Bayesian IRL can be very expensive in complex environments \citep{michini2012improving}. We could instead approximate \optr{} by using a less expensive IRL algorithm.  Furthermore, by Theorem \ref{thm:conv} we can guarantee convergence to optimal behavior.

Simpler choices for $\rest{}$ include the maximum-a-posteriori (MAP) estimate, which has previously been suggested as an alternative to Bayesian IRL \citep{choi2011map}, or the maximum likelihood estimate (MLE). If \human{} is noisily rational (Equation \ref{noisily-rational}) and $\beta = 1$, then the MLE is equivalent to Maximum Entropy IRL \citep{ziebart2008maximum}.

Although Theorem \ref{thm:conv} allows us to justify approximations at the limit, it is also important to ensure that \robot{}'s early behavior is not dangerous. Specifically, we may want \robot{} to err on the side of obedience early on. To investigate this we first prove a necessary property for any IRL-\robot{} to follow \human{}'s order:

\begin{lem}
\label{thm:undominated}
(Undominated necessary) Call $o_n$ \textit{undominated} if there exists $\rparams \in \Theta$ such that $o_n$ is optimal, i.e $o_n = \argmax_a{\rparams^T \phi(s_n, a)}$. It is necessary for $o_n$ to be undominated for an IRL-\robot{} to execute $o_n$.
\end{lem}
\begin{proof}
\robot{} executes $a = \argmax_{a} \rest_{n}^T\phi(s_n, a)$, so it is not possible for \robot{} to execute $o_n$ if there is no choice of $\rest{}_{n}$ that makes $o_n$ optimal. This can happen when one action dominates another action in value. For example, suppose $\rspace = \mathbb{R}^{2}$ and there are three actions with features $\phi(s, a_1) = [-1, -1]$, $\phi(s, a_2) = [0, 0]$, $\phi(s, a_3) = [1, 1]$. If \human{} picks $a_2$, then there is no $\rparams \in \rspace$ that makes $a_2$ optimal, and thus \robot{} will never follow $a_2$.
\end{proof}

One basic property we may want \robot{} to have is for it to listen to \human{} early on. The next theorem looks at we can guarantee about \robot{}'s obedience to the first order when \human{} is noisily rational.

\begin{figure*}[t]
\centering
\begin{subfigure}{\columnwidth}
\includegraphics[width=0.9\columnwidth]{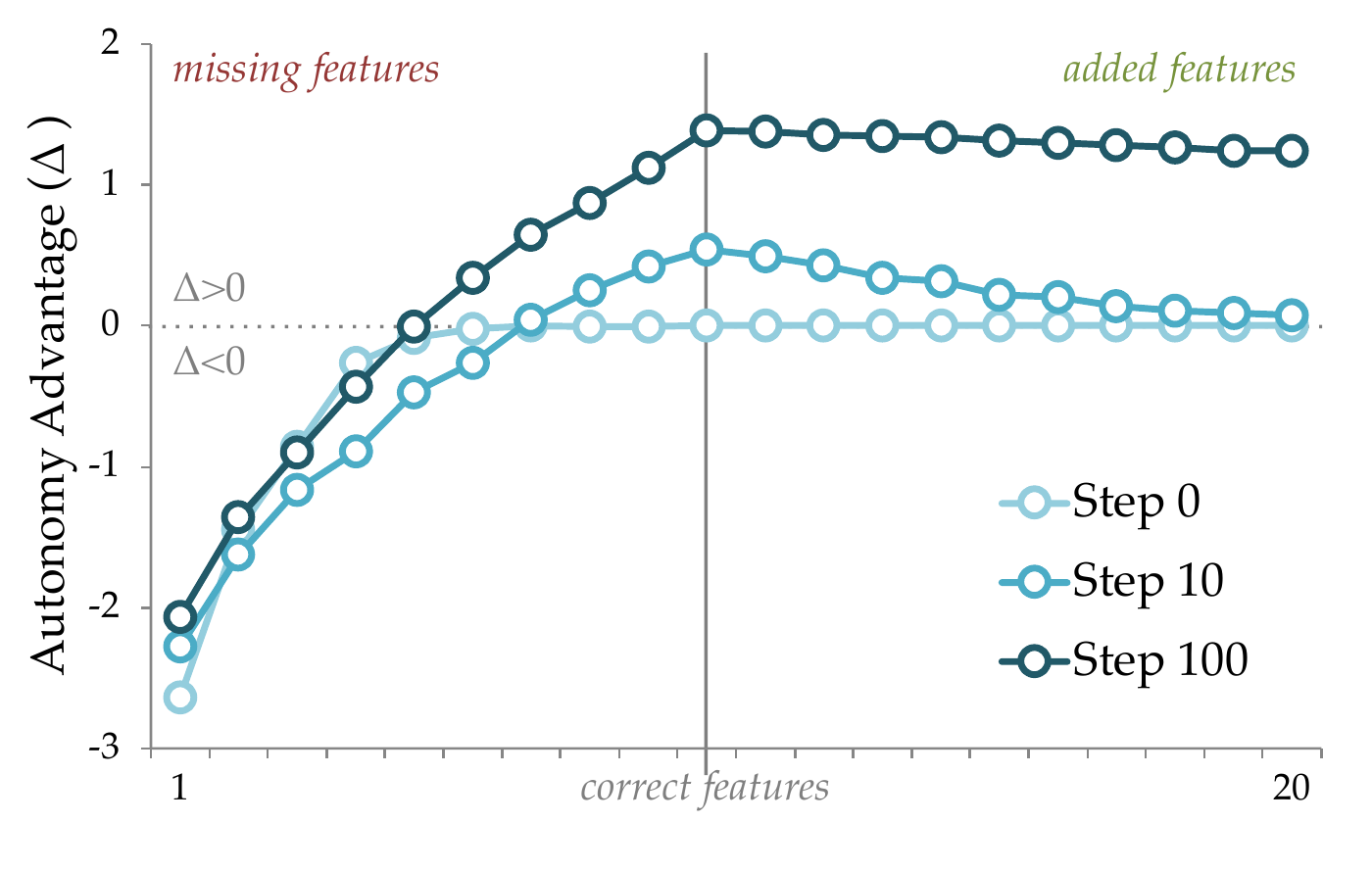}
\label{fig:}
\end{subfigure}
\begin{subfigure}{\columnwidth}
\includegraphics[width=0.9\columnwidth]{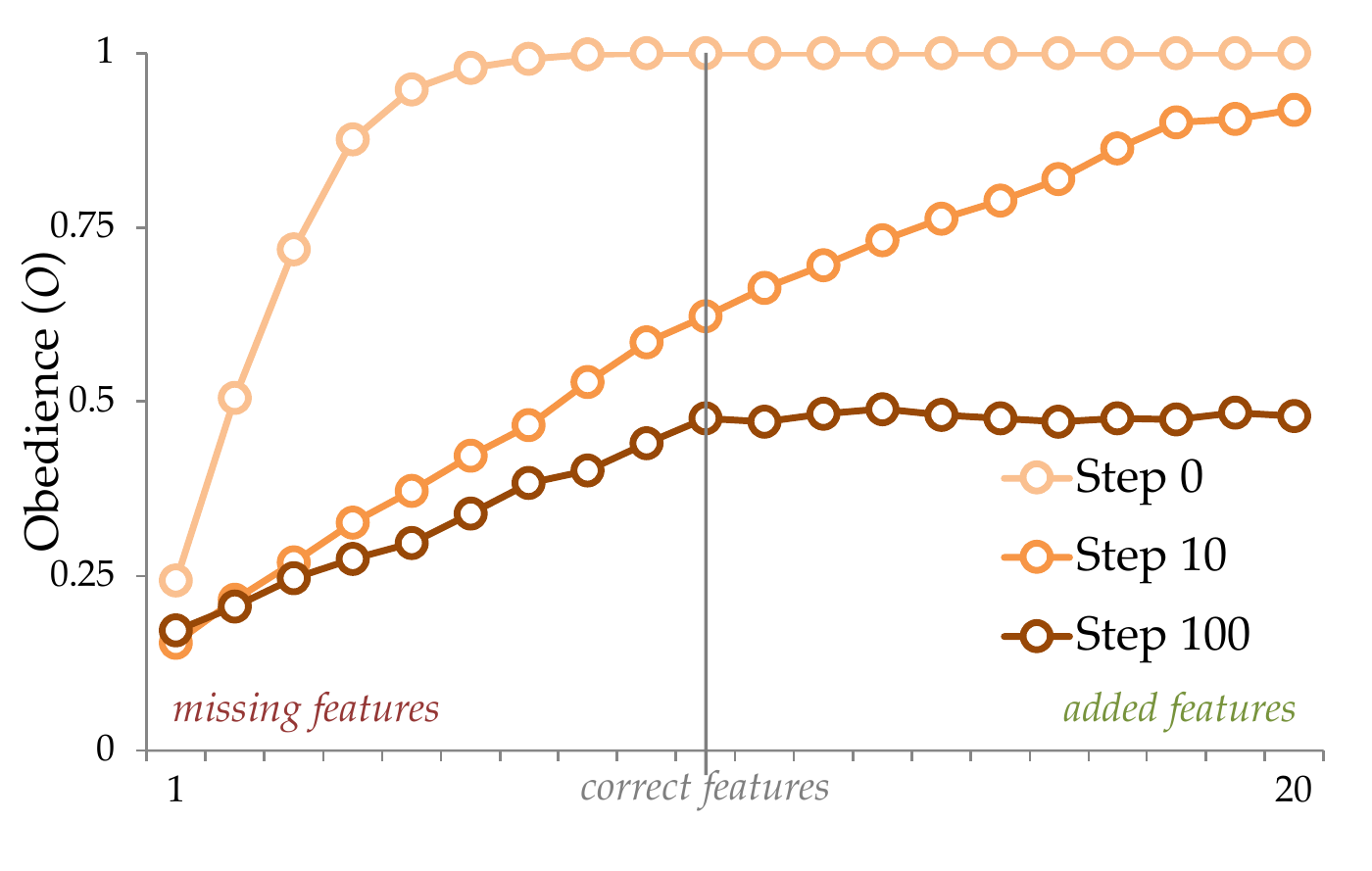}
\end{subfigure}
\caption{\adv{} and \ob{} when \rspace{} is misspecified}
\label{fig:missing_feats}
\end{figure*}

\begin{thm}
\label{thm:1order}
(Obedience to noisily rational \human{} on 1st order)
\begin{enumerate}[label=(\alph*)]
    \label{thm:pot-opt}
    \item When $\rspace = \mathbb{R}^{d}$ the MLE does not exist after one order. But if we constrain the norm of \rest{} to not be too large, then we can ensure that \robot{} follows an undominated $o_1$. In particular, $\exists K$ such that when \robot{} plans using the MLE $\rest{} \in \rspace' = \{\theta \in \rspace : ||\theta||_{2} \leq K \}$ \robot{} executes $o_1$ if and only if $o_1$ is undominated.
    \item If any IRL robot follows $o_1$, so does MLE-\robot{}. In particular, if \optr{} follows $o_1$, so does MLE-\robot{}.
    \item If \robot{} uses the MAP or posterior mean, it is not guaranteed to follow an undominated $o_1$. Furthermore, even if \optr{} follows $o_1$, MAP-\robot{} is not guaranteed to follow $o_1$.
\end{enumerate}
\end{thm}

\begin{proof}
\begin{enumerate}[label=(\alph*)]
    \item 
    The only if condition holds from Lemma \ref{thm:undominated}. Suppose $o_1$ is undominated. Then there exists $\theta^{*}$ such that $o_1$ is optimal for $\theta^{*}$. $o_1$ is still optimal for a scaled version, $c\theta^{*}$. As $c\rightarrow \infty$, $\nohp(o_1; c\theta^{*}) \rightarrow 1$, but never reaches it. Thus, the MLE does not exist.
    
    However since $\nohp(o_1; c\theta^{*})$ monotonically increases towards 1, $\exists C$ such that for $c > C$, $\nohp(o; c\theta^{*}) > 0.5$. If $K > C||\theta^{*}||$, then the MLE will be optimal for $o_1$ because $\nohp(o_1; \rest_1) \geq 0.5$ and \robot{} executes $a = \argmax_{a} \rest^T\phi(a) = \argmax_{a} \nohp(a; \rest)$. Therefore, in practice we can simply use the MLE while constraining $||\rparams||_{2}$ to be less than some very large number.
    
    \item From Lemma \ref{thm:undominated} if any IRL-\robot{} follows $o_1$, then $o_1$ is undominated. Then by (a) MLE-\robot{} follows $o_1$.

    \item For space we omit explicit counterexamples, but both statements hold because we can construct adversarial priors for which $o_1$ is suboptimal for the posterior mean and for which $o_1$ is optimal for the posterior mean, but not for the MAP.
\end{enumerate}
\end{proof}

Theorem \ref{thm:1order} suggests that at least at the beginning when \robot{} uses the MLE it errs on the side of giving us the ``benefit of the doubt", which is exactly what we would want out of an approximation. 

Figure \ref{fig:adv} and \ref{fig:obedience} plot \adv{} and \ob{} for an IRL robot that uses the MLE. As expected, \robot{} gains more reward than a blindly obedient one ($\adv > 0$), eventually converging to the maximum autonomy advantage (Figure \ref{fig:adv}). On the other hand, as \robot{} learns about $\rparams$, its obedience also decreases, until eventually it only listens to the human when she gives the optimal order (Figure \ref{fig:obedience}). 

As pointed out in Remark \ref{rationality-adv}, \adv{} is eventually higher for more irrational humans. However, a more irrational human also provides noisier evidence of \rparams, so the rate of convergence of $\adv{}$ is also slower. So, although initially \adv{} may be lower for a more irrational \human{}, in the long run there is more to gain from being autonomous when interacting with a more irrational human. Figure \ref{fig:adv-rat} shows this empirically.

All experiments in this paper use the following parameters unless otherwise noted. At the start of each episode $\rparams \sim \mathcal{N}(0, I)$ and at each step $\phi_n(a) \sim \mathcal{N}(0, I)$. There are 10 actions, 10 features, and $\beta = 2$. \footnote{All experiments can be replicated using the Jupyter notebook available at \url{http://github.com/smilli/obedience}}

Finally, even with good approximations we may still have good reason for feeling hesitation about disobedient robots. The naive analysis presented so far assumes that \robot{}'s models are perfect, but it is almost certain that \robot{}'s models of complex things like human preferences and behavior will be incorrect. By Theorem \ref{thm:undominated}, \robot{} will not obey even the first order made by \human{} if there is no $\rparams{} \in \rspace{}$ that makes \human{}'s order optimal. So clearly, it is possible to have disastrous effects by having an incorrect model of \rspace{}. In the next section we look at how misspecification of possible human preferences (\rspace{}) and human behavior (\hpolicy{}) can cause the robot to be overconfident and in turn less obedient than it should be. \textit{The autonomy advantage can easily become the rebellion regret.}

\section{Model Misspecification}
\label{sec:wrong-model}

\prg{Incorrect Model of Human Behavior}
Having an incorrect model of \human{}'s rationality ($\beta$) does not change the actions of MLE-\robot{}, but does change the actions of \optr{}.

\begin{thm}
\label{thm:incorrectpolicy}
(Incorrect model of human policy) Let $\beta^{0}$ be \human{}'s true rationality and $\beta'$ be the rationality that \robot{} believes \human{} has. Let $\rest$ and $\rest'$ be \robot{}'s estimate under the true model and misspecified model, respectively. Call \robot{} \textbf{robust} if its actions under $\beta'$ are the same as its actions under $\beta^{0}$.

\begin{enumerate}[label=(\alph*)]
\item MLE-\robot{} is robust.
\item \optr{} is not robust.
\end{enumerate}
\end{thm}

\begin{proof}
    \begin{enumerate}[label=(\alph*)]
    \item The log likelihood $l(\hist| \rparams)$ is concave in $\eta = \rparams/\beta$. So, $\rest'_{n} = (\beta'/\beta^{0})\rest_{n}$. This does not change \robot{}'s action: $\argmax_{a} \rest_{n}^{'T} \phi_n(a) = \argmax_{a} \rest_{n}^T \phi_n(a)$
    
    \item Counterexamples can be constructed based on the fact that as $\beta \rightarrow 0$, \human{} becomes rational, but as $\beta \rightarrow \infty$, \human{} becomes completely random. Thus, the likelihood will ``win" over the prior for $\beta \rightarrow 0$, but not when $\beta \rightarrow \infty$.
    \end{enumerate}

\end{proof}

MLE-\robot{} is more robust than the optimal \optr{}. This suggests a reason beyond computational savings for using approximations: the approximations may be more robust to misspecification than the optimal policy.

\begin{rem}
Theorem \ref{thm:incorrectpolicy} may give us insight into why Maximum Entropy IRL (which is the MLE with $\beta = 1$) works well in practice. In simple environments where noisy rationality can be used as a model of human behavior, getting the level of noisiness right doesn't matter.
\end{rem}

\prg{Incorrect Model of Human Preferences}
The simplest way that \human{}'s preferences may be misspecified is through the featurization of \rparams{}. Suppose $\rparams{} \in \rspace = \mathbb{R}^{d}$. \robot{} believes that $\rspace = \mathbb{R}^{d'}$. \robot{} may be missing features ($d' < d$) or may have irrelevant features ($d' > d$). \robot{} observes a $d'$ dimensional feature vector for each action: $\phi_n(a) \sim N(0, I^{d' \times d'})$. The true $\rparams{}$ depends on only the first $d$ features, but \robot{} estimates $\rparams{} \in \mathbb{R}^{d'}$. Figure \ref{fig:missing_feats} shows how \adv{} and \ob{} change over time as a function of the number of features for a MLE-\robot{}. When \robot{} has irrelevant features it still achieves a positive \adv{} (and still converges to the maximum \adv{} because $\rest{}$ remains consistent over a superset of $\rspace$). But if \robot{} is missing features, then $\adv$ may be negative, and thus \robot{} would be better off being blindly obedient instead. Furthermore, when \robot{} contains extra features it is more obedient than it would be with the true model. But if \robot{} is missing features, then it is less obedient than it should be. This suggests that to ensure \robot{} errs on the side of obedience we should err on the side of giving \robot{} a \textit{more} complex model.

\prg{Detecting Misspecification}
If \robot{} has the wrong model of \rspace{}, \robot{} may be better off being obedient. In the remainder of this section we look at how \robot{} can detect that it is missing features and act accordingly obedient.

\begin{figure}[t]
\centering
\includegraphics[width=\columnwidth]{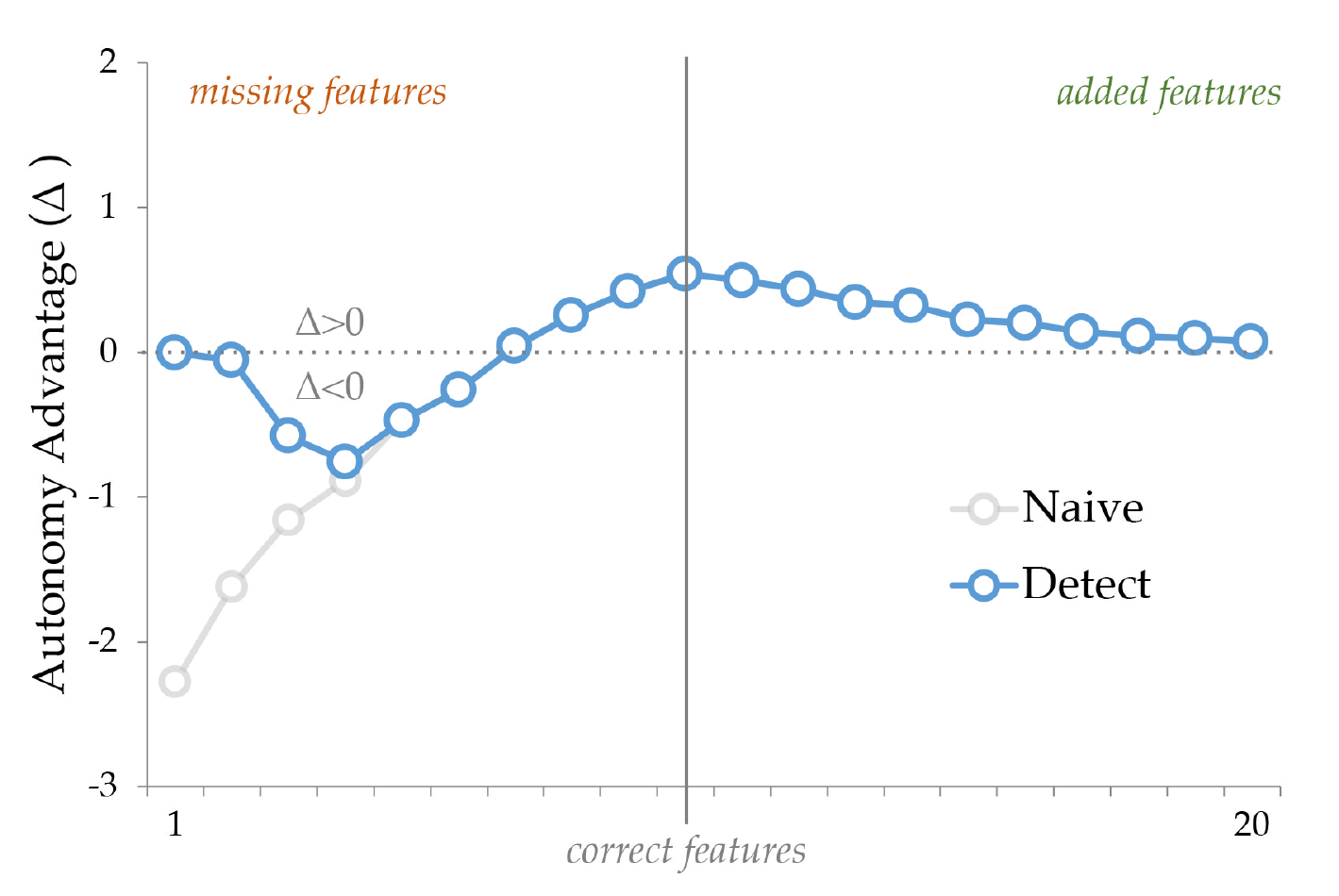}
\caption{\textit{(Detecting misspecification)} The bold line shows the \robot{} that tries detecting missing features (Equation \ref{detect-policy}), as compared to MLE-\robot{} (which is also shown in Figure \ref{fig:missing_feats}).}
\label{fig:detect_misspec}
\end{figure}

\begin{rem}
\label{policy-mixing}
(Policy mixing) We can make \robot{} more obedient, while maintaining convergence to the maximum advantage, by mixing \robot{}'s policy $\rpolicy^{I}$ with a blindly obedient policy:
\begin{equation*}
\rpolicy(\hist) = 1\{ \delta_n = 0 \} \bopolicy(\hist) + 1\{\delta_n = 1\} \rpolicy^{I}(\hist)
\end{equation*}
\begin{equation*}
P(\delta_n=i) = \begin{cases}
c_{n} & i = 0 \\
1 - c_{n} & i = 1
\end{cases}
\end{equation*}
where $1 \geq c_{n} \geq 0$ with $c_{n} \rightarrow 0$. In particular, we can have an initial ``burn-in" period where \robot{} is blindly obedient for a finite number of rounds before switching to $\rpolicy^{I}$.
\end{rem}

By Theorem \ref{thm:1order} we know MLE-\robot{} will always obey \human{}'s first order if it is undominated. This means that for MLE-\robot{}, $\ob_{1}$ should be close to one if undominated orders are expected to be rare. As pointed out in Remark \ref{policy-mixing} we can have an initial ``burn-in" period where \robot{} always obeys \human{}. Let \robot{} have a burn-in obedience period of $B$ rounds. \robot{} uses this burn-in period to calculate the sample obedience on the first order:

$$\tilde{\ob}_1 = \frac{1}{B} \sum_{i=1}^{B} 1\{ \argmax_{a} \rest_{1}(h_i) ^T\phi_i(a) = o_i \}$$

If $\tilde{\ob}_1$ is not close to one, then it is likely that \robot{} has the wrong model of \rspace{}, and would be better off just being obedient. So, we can choose some small $\epsilon$ and make \robot{}'s policy
\begin{align}
\label{detect-policy}
\rpolicy(\hist) = 
\begin{cases}
o_n & n \leq B \\
o_n & n > B, \tilde{\ob}_{1} < 1 -  \epsilon \\
\argmax_{a} \rest_{n}^T\phi_n(a) & n > B, \tilde{\ob}_{1} > 1 - \epsilon
\end{cases}
\end{align}

Figure \ref{fig:detect_misspec} shows the \adv{} of this robot as compared to the MLE-\robot{} from Figure \ref{fig:missing_feats} after using the first ten orders as a burn-in period. This \robot{} achieves higher \adv{} than MLE-\robot{} when missing features and still does as well as MLE-\robot{} when it isn't missing features. 

Note that this strategy relies on the fact that MLE-\robot{} has the property of always following an undominated first order. If \robot{} were using the optimal policy, it is unclear what kind of simple property we could use to detect missing features. This gives us another reason for using an approximation: we may be able to leverage its properties to detect misspecification.

\section{Related Work}

\prg{Ensuring Obedience}
There are several recent examples of research that aim to provably ensure that \human{} can interrupt \robot{}. \citep{soares2015corrigibility,Orseau2016SafelyIA,hadfield2016off}. \citet{hadfield2016off} show that \robot{}'s obedience depends on a tradeoff between \robot{}'s uncertainty about \rparams{} and \human{}'s rationality. However, they considered \robot{}'s uncertainty in the abstract. In practice \robot{} would need to learn about \rparams{} \textit{through} \human{}'s behavior. Our work analyzes how the way \robot{} learns about \rparams{} impacts its performance and obedience.

\prg{Intent Inference For Assistance}
 Instead of just being blindly obedient, an autonomous system can infer \human{}'s intention and actively assist \human{} in achieving it. Do What I Mean software packages interpret the intent behind what a programmer wrote to automatically correct programming errors  \citep{teitelman19707}. When a user uses a telepointer network lag can cause jitter in her cursor's path. \citet{gutwin2003using} address this by displaying a prediction of the user's desired path, rather than the actual cursor path. 
 
 Similarly, in assistive teleoperation, the robot does not directly execute \human{}'s (potentially noisy) input. It instead acts based on an inference of \human{}'s intent. In \cite{dragan2012formalizing} \robot{} acts according to an arbitration between \human{}'s policy and \robot{}'s prediction of \human{}'s policy. Like our work, \cite{javdani2015shared} formalize assistive teleoperation as a POMDP in which \human{}'s goals are unknown, and try to optimize an inference of \human{}'s goal. While assistive teleoperation apriori assumes that \robot{} should act assistively, we show that under model misspecification sometimes it is better for \robot{} to simply defer to \human{}, and contribute a method to decide between active assistance and blind obedience (Remark \ref{policy-mixing}).

\prg{Inverse Reinforcement Learning}
We use inverse reinforcement learning \citep{ng2000algorithms,abbeel2004apprenticeship} to infer \rparams{} from \human{}'s orders. We analyze how different IRL algorithms affect autonomy advantage and obedience, properties not previously studied in the literature. In addition, we analyze how model misspecification of the features of the space of reward parameters or the \human{}'s rationality impacts autonomy advantage and obedience. 

IRL algorithms typically assume that \human{} is rational or noisily rational. We show that Maximum Entropy IRL \citep{ziebart2008maximum} is robust to misspecification of a noisily rational \human{}'s rationality ($\beta$). However, humans are not truly noisily rational, and in the future it is important to investigate other models of humans in IRL and their potential misspecifications. \cite{evans2016learning} takes a step in this direction and models \human{} as temporally inconsistent and potentially having false beliefs. In addition, IRL assumes that \human{} acts without awareness of \robot{}'s presence, \textit{cooperative inverse reinforcement learning} \citep{hadfield2016cooperative} relaxes this assumption by modeling the interaction between \human{} and \robot{} as a two-player cooperative game.

\section{Conclusion}
To summarize our key takeaways:
\begin{enumerate}
\item ($\adv{} > 0$) If \human{} is not rational, then \robot{} can always attain a positive \adv{}. Thus, forcing \robot{} to be blindly obedient requires giving up on a positive $\adv{}$.
\item ($\adv{}$ vs $\ob{}$) There exists a tradeoff between $\adv{}$ and $\ob{}$. At the limit \optr{} attains the maximum $\adv{}$, but only obeys \human{}'s order when it is the optimal action.
\item (MLE-\robot{}) When \human{} is noisily rational MLE-\robot{} is at least as obedient as any other IRL-\robot{} to \human{}'s first order. This suggests that the MLE is a good approximation to \optr{} because it errs on the side of obedience.
\item (Wrong $\beta$) MLE-\robot{} is robust to having the wrong model of the human's rationality ($\beta$), but \optr{} is not. This suggests that we may not want to use the ``optimal" policy because it may not be very robust to misspecification.
\item (Wrong $\rspace$) If \robot{} has extra features, it is more obedient than with the true model, whereas if it is missing features, then it is less obedient. If \robot{} has extra features, it will still converge to the maximum \adv{}. But if \robot{} is missing features, it is sometimes better for \robot{} to be obedient. This implies that erring on the side of extra features is far better than erring on the side of fewer features.
\item (Detecting wrong $\rspace$) We can detect missing features by checking how likely MLE-\robot{} is to follow the first order.
\end{enumerate}

Overall, our analysis suggests that in the long-term we should aim to create robots that intelligently decide when to follow orders, but in the meantime it is crucial to ensure that these robots err on the side of obedience and are robust to misspecified models.

\section{Acknowledgements}
We thank Daniel Filan for feedback on an early draft.


\input{ijcai17.bbl}
\end{document}

%% file: variables.tex
\usepackage{accents}
\usepackage{stackengine}
\newcommand{\ubar}[1]{\stackunder[1.2pt]{$#1$}{\rule{.8ex}{.075ex}}}

\newcommand{\human}{\ensuremath{\mathbf{H}}}
\newcommand{\robot}{\ensuremath{\mathbf{R}}}
\newcommand{\optr}{\ensuremath{\robot^{*}}}
\newcommand{\ob}{\ensuremath{\mathcal{O}}}
\newcommand{\adv}{\ensuremath{\Delta}}
\newcommand{\rest}{\ensuremath{\hat{\rparams}}}
\newcommand{\rspace}{\ensuremath{\Theta}}
\newcommand{\hpolicy}{\ensuremath{\pi_{\mathbf{H}}}}
\newcommand{\rpolicy}{\ensuremath{\pi_{\mathbf{R}}}}
\newcommand{\opthp}{\ensuremath{\hpolicy^{*}}}
\newcommand{\nohp}{\ensuremath{\tilde{\pi}_{\mathbf{H}}}}
\newcommand{\optrp}{\ensuremath{\rpolicy^{*}}}
\newcommand{\bopolicy}{\ensuremath{\pi^{O}_{\mathbf{R}}}}
\newcommand{\hactions}{\ensuremath{\mathcal{A}^{\human}}}
\newcommand{\ractions}{\ensuremath{\mathcal{A}^{\robot}}}
\newcommand{\hist}{\ensuremath{h}}

\newcommand{\rparams}{\ensuremath{\theta}}

\newtheorem{thm}{Theorem}
\newtheorem{lem}{Lemma}

\newtheorem{rem}{Remark}

{\begin{list}{$\bullet$}{%
    \setlength{\topsep}{0in}
    \setlength{\partopsep}{0in}
    \setlength{\itemsep}{0in}
    \setlength{\parsep}{0in}
    \setlength{\leftmargin}{3.5em}
    \setlength{\rightmargin}{0in}
    \setlength{\itemindent}{-.1in}
}
}%
{\end{list}
}

\DeclareMathOperator*{\argmax}{argmax}

\newcommand{\prg}[1]{\noindent\textbf{#1.}}